\documentclass[letterpaper]{article} % DO NOT CHANGE THIS
\usepackage{aaai2026}  % DO NOT CHANGE THIS
\usepackage{times}  % DO NOT CHANGE THIS
\usepackage{helvet}  % DO NOT CHANGE THIS
\usepackage{courier}  % DO NOT CHANGE THIS
\usepackage[hyphens]{url}  % DO NOT CHANGE THIS
\usepackage{graphicx} % DO NOT CHANGE THIS
\usepackage{natbib}  % DO NOT CHANGE THIS AND DO NOT ADD ANY OPTIONS TO IT
\usepackage{caption} % DO NOT CHANGE THIS AND DO NOT ADD ANY OPTIONS TO IT
\usepackage{algorithm}
\usepackage{algorithmic}

\usepackage{newfloat}
\usepackage{listings}
\DeclareCaptionStyle{ruled}{labelfont=normalfont,labelsep=colon,strut=off} % DO NOT CHANGE THIS
\floatstyle{ruled}
\newfloat{listing}{tb}{lst}{}
\floatname{listing}{Listing}
\usepackage{subcaption}
\usepackage{amssymb}
\usepackage{enumitem}
\usepackage{amsmath,amsfonts,amsthm}
\usepackage{bm}
\usepackage{mathtools}
\usepackage{soul} % \hl
\usepackage{array}
\usepackage{multirow}

\usepackage{nicefrac}
\usepackage{microtype}      % microtypography
\usepackage{xcolor}         % colors

\newtheorem{theorem}{Theorem}
\newtheorem{lemma}{Lemma}
\newtheorem{remark}{Remark}

\ifodd 1

\newcommand{\com}[1]{\textbf{\color{red}(comment: #1)}} %comment of the text
\newcommand{\res}[1]{\textbf{\color{magenta}(RESPONSE: #1)}} %response to comments
\else

\newcommand{\com}[1]{}
\newcommand{\res}[1]{}
\fi

\title{Stabilizing Self-Consuming Diffusion Models with Latent Space Filtering}
\author {
    Zhongteng Cai\textsuperscript{\rm 1},
    Yaxuan Wang\textsuperscript{\rm 2},
    Yang Liu\textsuperscript{\rm 2},
    Xueru Zhang\textsuperscript{\rm 1}
}
\affiliations {
    \textsuperscript{\rm 1}
    The Ohio State University\\
    \textsuperscript{\rm 2}
    University of California, Santa Cruz\\
    cai.1125@osu.edu, ywan1225@ucsc.edu, yangliu@ucsc.edu, zhang.12807@osu.edu
}

\begin{document}

\maketitle

\begin{abstract}
As synthetic data proliferates across the Internet, it is often reused to train successive generations of generative models. This creates a ``self-consuming loop" that can lead to training instability or \textit{model collapse}. Common strategies to address the issue---such as accumulating historical training data or injecting fresh real data---either increase computational cost or require expensive human annotation. In this paper, we empirically analyze the latent space dynamics of self-consuming diffusion models and observe that the low-dimensional structure of latent representations extracted from synthetic data degrade over generations. Based on this insight, we propose \textit{Latent Space Filtering} (LSF), a novel approach that mitigates model collapse by filtering out less realistic synthetic data from mixed datasets. Theoretically, we present a framework that connects latent space degradation to empirical observations. Experimentally, we show that LSF consistently outperforms existing baselines across multiple real-world datasets, effectively mitigating model collapse without increasing training cost or relying on human annotation.
\end{abstract}

\section{Introduction}
\label{sec:intro}

Diffusion models have emerged as a leading class of generative models in computer vision, achieving state-of-the-art performance in image generation tasks~\cite{beat_gan}. By learning to reverse a noise-injection process, these models can synthesize high-quality samples~\cite{ddpm, estimate_grad}. Compared to earlier approaches such as variational autoencoders (VAEs)~\cite{vae} and generative adversarial networks (GANs)~\cite{gan}, diffusion models offer greater training stability, improved sample fidelity~\citep{beat_gan}, and broader applicability across diverse domains. %including audio~\citep{audio} and natural language generation~\citep{language}.

Modern diffusion models are typically trained on large-scale datasets scraped from the Internet. However, as synthetic data proliferates online, it inevitably becomes part of the training data for future generations of models, creating a ``self-consuming training loop." Recent studies have shown that such self-consuming loops may lead to model collapse~\cite{recrusion,accumulate,mad,collapse,wei2025selfconsuming}, training instability~\cite{stable}, and potential bias amplification~\cite{amplification,FairnessFeedback,xie2024automating}. 
To mitigate these issues, several strategies have been proposed, including: (i) accumulating all historical (real and synthetic) samples into the training dataset \citep{accumulate,news_bias_amplify}; (ii) injecting sufficient fresh real data during each generation cycle~\citep{mad,remove,recrusive_stability}; (iii) modifying the training process in each round---for example, by utilizing prior knowledge of the real distribution~\citep{self_correction}, merging sequentially trained models~\citep{sims,self_improve_merging}, or leveraging feedback on synthetic data~\citep{verification}. However, these approaches often incur substantial storage and computational costs, require expensive and reliable human annotations, or are limited to single-step retraining rather than the full multi-step self-consuming process.

%However, these datasets increasingly contain synthetic content generated by generative models~\citep{internet}. At the same time, access to fresh and clean real-world data has become more constrained, since public training data is being exhausted. As a result, \textit{self-consuming}, which means training a model using its own generated samples, have been increasingly explored. Recent works have shown that generative models trained on their own outputs often suffer from performance degradation, producing outputs that are more repetitive, with lower fidelity or diversity~\citep{recrusion, mad}. This phenomenon is referred to as \textit{model collapse}, which poses threats to the future training of generative models.

In this work, we address the above limitations and propose an alternative strategy to mitigate model collapse without introducing additional storage overhead, training cost, or reliance on human annotations. Unlike existing works that study self-consuming generative models in input space, we shift the focus to latent space, aiming to understand the causes of model collapse through the lens of latent representations. The idea is to first investigate how the representations of model-generated data evolve in latent space across successive self-consuming training loops. Based on this insight, we develop a filtering mechanism that mitigates model collapse by identifying and removing less realistic data from the mixed training set at each round of the self-consuming process.

 %motivated by a recent study \citep{dynamic} that highlights diffusion models have been recognized as an efficient representation learner, 
Specifically, we use diffusion models to extract latent representations for real and synthetic images sampled at each round of the self-consuming training loop \citep{dynamic}. For each generation, we compute Orthogonal Low-rank Embedding (OLE) scores~\citep{ole} of the latent representations, which quantify the orthogonality of class-specific subspaces and reflect the underlying low-dimensional structure. Our key observation is that the low-dimensional structure of latent representations from synthetic images gradually degenerates as self-consuming training progresses, whereas real images tend to exhibit more orthogonal subspaces. % have the most orthogonal subspaces. This suggests a correlation between low-dimensional structure in the latent space and the quality of the generated samples.
Motivated by this observation, we introduce \emph{Latent Space Filtering} (LSF), a method that identifies and filters synthetic images with degraded latent representations before they are used in further training. Specifically, we train a representation probing classifier for real images based on their latent embeddings. The confidence score from this classifier serves as a proxy for how well a data sample aligns with the low-dimensional manifold of real data. Samples with low alignment are discarded, preventing them from introducing additional noise into the model's learned latent subspaces. In addition to the algorithmic development, we establish a theoretical foundation for the proposed method, which explains how the degeneration of low-dimensional structures correlates with changes in both the OLE scores and the probing classifier's confidence scores. 

Our main contributions are summarized below. More related work is discussed in Appendix~\ref{related}. Conclusions and limitations are provided in Appendix~\ref{appendix:conclusion}.

% \begin{itemize}[leftmargin=*,topsep=0cm,itemsep=0cm]
\begin{itemize}
    \item In Section~\ref{sec:background}, we formalize the problem of studying self-consuming diffusion models in latent space. To the best of our knowledge, this is the first work to investigate how the low-dimensional structure of latent representations evolves across the self-consuming training loop.
    \item In Section~\ref{sec:observation}, we analyze how the low-dimensional structure of latent representations---measured by Orthogonal Low-rank Embedding (OLE) scores---evolves across successive generations of self-consuming diffusion models.
    
  %  We characterize the progressive degeneration of low-dimensional latent structures in self-consuming diffusion models.
    \item Motivated by empirical observations, Section~\ref{sec:method} proposes \textit{Latent Space Filtering} (LSF), a practical approach that mitigates model collapse by filtering low-quality synthetic data based on subspace alignment, without requiring additional real data collection or increased training cost.%Motivated by the empirical observations, we propose \textit{Latent Space Filtering} (LSF) in Section~\ref{sec:method}, a practical method that mitigates model collapse by filtering low-quality synthetic data using subspace alignment, without relying on real data collection or increased training budgets.  
    \item We establish a theoretical framework to justify both empirical observations and the proposed method. This framework explains how the degeneration of low-dimensional structures correlates with changes in OLE scores (Theorem~\ref{theorem:ole_lob}) and in the confidence scores of the probing classifier used to assess subspace alignment (Theorem~\ref{theorem:score}).

%latent space degeneration affects the dynamics of evaluated metrics, which justifies our observations and the proposed method.
    \item In Section~\ref{sec:exp}, we extensively evaluate the proposed method across multiple real-world datasets\footnote{Source code available at \url{https://github.com/osu-srml/Latent-Space-Filtering}}. The results demonstrate that LSF effectively mitigates model collapse and outperforms existing baselines, without the need for a growing training set or additional curated data.
    
    %empirically demonstrate the effectiveness of LSF in mitigating model collapse compared with other baselines. Results show that LSF can achieve better stability and fidelity under a fixed training budget.
\end{itemize}
\section{Problem formulation}
\label{sec:background}
Consider a platform that iteratively trains diffusion models using a sequence of training datasets. Let $f^{(k)}$ denote the diffusion model trained on dataset $\mathcal{D}^{(k)}$ during the $k$-th round of the iterative retraining loop, and let $\widehat{\mathcal{D}}^{(k+1)}$ be the synthetic dataset generated by $f^{(k)}$. Without loss of generality, we assume the initial dataset $\mathcal{D}^{(0)}$ used to train the first model $f^{(0)}$ consists of real data. We consider two variants of the self-consuming training loop \cite{mad}:
% \begin{enumerate}[leftmargin=*,topsep=0cm,itemsep=0cm]
\begin{enumerate}
    \item \textbf{Pure synthetic loop:} Only the synthetic data generated by the most recent model is used to train the next model, i.e., ${\mathcal{D}}^{(k)}=\widehat{\mathcal{D}}^{(k)}, \forall k$.
   \item \textbf{Accumulation loop:} All previously used training data, including the initial real dataset $\mathcal{D}^{(0)} $ and synthetic samples from earlier generations, are accumulated to train the next model, i.e., $ \mathcal{D}^{(k)} \coloneqq \mathcal{D}^{(k-1)} \cup \widehat{\mathcal{D}}^{(k)} = \mathcal{D}^{(0)} \cup \widehat{\mathcal{D}}^{(1)} \cup \cdots \cup \widehat{\mathcal{D}}^{(k)}, \forall k$.
\end{enumerate}
In practice and the experiments presented in Section~\ref{sec:exp}, we may apply additional processing steps to  $\mathcal{D}^{(k)}$, such as re-sampling or filtering, before training $f^{(k)}$.  

\smallskip
\noindent \textbf{Background: Diffusion Model.} We use Denoising Diffusion Probabilistic Models (DDPM)~\citep{ddpm} as an example to illustrate the training of $f^{(k)}$ in each round of the retraining loop, though our analysis and method are broadly applicable beyond this specific architecture.

Training a DDPM consists of two stages: a \textit{forward process} that gradually adds noise to the data, and a \textit{reverse process} that learns to remove the noise. In the forward process, Gaussian noise is iteratively added to a clean image $\bm{x}_0\sim q(\bm{x}_0)$ over $T$ steps, producing a sequence of increasingly noisy images $\{\bm{x}_1, \dots, \bm{x}_T\}$. The conditional distribution at each timestep is defined as:
\[
q(\bm{x}_t|\bm{x}_{t-1}) = \mathcal{N}(\bm{x}_t; \sqrt{1 - \beta_t} \bm{x}_{t-1}, \beta_t \bm{I}),
\]
where $\beta_t$ is a variance schedule. As $T \to \infty$, the image $\bm{x}_T$ approaches pure Gaussian noise. The reverse process learns to denoise $\bm{\epsilon}_0 \sim \mathcal{N}(\bm{0}, \bm{I})$ back to $\bm{x}_0$ using a model $\widehat{\bm{\epsilon}}_{\theta}$ trained to predict the noise at each step. Specifically, the model is optimized by minimizing the simplified Evidence Lower Bound (ELBO) objective: $\nabla_{\theta}\left\| \widehat{\bm{\epsilon}}_\theta(\bm{x}_t, t) - \bm{\epsilon}_0 \right\|^2
$. At generation time, denoising proceeds from
$\bm{x}_T \sim \mathcal{N}(\bm{0}, \bm{I})$ to $\bm{x}_0$ using the learned $\widehat{\bm{\epsilon}}_{\theta}$ according to the following: 
\[
\bm{x}_{t-1} = \frac{1}{\sqrt{\alpha_t}} \left(\bm{x}_t - \frac{1 - \alpha_t}{\sqrt{1 - \bar{\alpha}_t}} \widehat{\bm{\epsilon}}_\theta(\bm{x}_t, t)\right) + \sigma_q(t) \bm{z},
\]
where $\alpha_t = 1 - \beta_t$, $\bar{\alpha}_t = \prod_{i=1}^t \alpha_i$, $\bm{z} \sim \mathcal{N}(\bm{0}, \bm{I})$, and
$\sigma_q(t) = \sqrt{\frac{(1 - \alpha_t)(1 - \bar{\alpha}_{t-1})}{1 - \bar{\alpha}_t}}$.
\citet{estimate_grad} propose a related framework based on score matching, where a neural network is trained to approximate the score function $\bm{s}_\theta(\bm{x}) \approx \nabla_{\bm{x}} \log q(\bm{x})$. Using Langevin dynamics~\citep{langevin}, one can then generate samples that follow the target distribution via the learned score function. This framework connects to DDPM through the following approximation: $\bm{s}_\theta(\bm{x}_t,t) \approx - \frac{\widehat{\bm{\epsilon}}_\theta(\bm{x}_t, t)}{\sqrt{1 - \bar{\alpha}_t}}$. 

In practice, diffusion models typically employ U-Net architectures~\citep{unet}, which encode features through a series of downsampling layers and reconstructs the image via upsampling. We extract latent representations from U-Net’s bottleneck layer, which has proven effective for various downstream tasks such as controllable image editing~\citep{semantic_latent,dynamic}.

\smallskip
\noindent\textbf{Low-dimensional structure of latent representations under self-consuming loop.} 
Existing research on self-consuming generative models has primarily focused on the input space, while the impact on latent representations remains largely unexplored. Understanding how latent representations and their underlying low-dimensional structure evolve during self-consuming training may offer valuable insights for developing strategies to mitigate model collapse.

%As shown in a recent study~\citep{dynamic}, diffusion models can extract informative latent representations that are useful for various downstream tasks such as classification, image editing, and semantic segmentation~\citep{semantic_latent}. Therefore, we extract the latent representation of data using a diffusion model. 
Specifically, for the $k$-th round of iterative retraining loop, let the denoising model of the diffusion model $f^{(k)}$ be denoted as $\widehat{\bm{\epsilon}}^{(k)}_{\theta}(\cdot) = d^{(k)} \circ e^{(k)}(\cdot)$, which is formed by an encoder $e^{(k)}$ and a decoder $d^{(k)}$. Since $e^{(0)}$ is the encoder of the initial model $f^{(0)}$ trained solely on \textit{real} data ${\mathcal{D}}^{(0)}$, we consider the output of $e^{(0)}$ as the \textbf{latent representation} of a data sample in this work. For a sample $\bm{x}^{(k)}$ generated by $f^{(k)}$ at $k$-th round of self-consuming training loop, we denote the latent representation at denoising timestep $t$  as $\bm{h}^{(k)}_t = e^{(0)}(\bm{x}^{(k)},t)$, following prior works on representation extraction via diffusion models~\cite{dynamic}.

\smallskip
\noindent\textbf{Objectives.} %Recent works have highlighted that generative models trained under a self-consuming loop may lead to model collapse, where the model degrades in both fidelity and diversity. Existing mitigating strategies, such as accumulating historical training data or injecting fresh real data, either increase computational cost or require expensive human annotations. This motivates the present work, where 
In this work, we focus on self-consuming diffusion models and aim to leverage the low-dimensional structure of latent representations to better understand the causes of model collapse and develop effective mitigation strategies. Our goals are to: 1) examine how the low-dimensional structure of latent representations $\bm{h}^{(k)}_t$ evolves  throughout the self-consuming training loop; and 2) based on these insights, develop a novel mechanism to mitigate model collapse by filtering out less realistic samples, which compared to prior methods, without incurring additional training costs or relying on human annotations.

\section{Evolution of low-dimensional structure of latent representations}
\label{sec:observation}

\textbf{Empirical observations.}
\label{subsec:observation}
For a batch of $N$ images $\mathcal{B}^{(k)} = \{(\bm{x}^{(k)}_1, y^{(k)}_1), \dots, (\bm{x}^{(k)}_N, y^{(k)}_N)\}$ sampled from dataset ${\mathcal{D}}^{(k)}$ at $k$-th round of self-consuming loop,\footnote{The class label can either be the ground-truth annotation or a group label assigned by an unsupervised clustering algorithm. Unsupervised clustering enables us to explore low-dimensional structures beyond predefined classes, as diffusion models can capture factorized features that may not align with standard semantic categories.} the latent representations extracted using $e^{(0)}$ in matrix form are given by:
\begin{align}\label{eq:represntaton}
\textstyle \bm{M}^{(k)}_t = 
&\begin{bmatrix}
%\vert & \vert &        & \vert \\
\bm{h}^{(k)}_{1,t} & \bm{h}^{(k)}_{2,t} & \cdots & \bm{h}^{(k)}_{N,t} %\\
%\vert & \vert &        & \vert \\
\end{bmatrix},~~ \notag\\
 \text{ where } ~~\bm{h}^{(k)}_{i,t} = &~e^{(0)}(\bm{x}^{(k)}_i,t)
\end{align}
Similarly, we define the \textit{class-specific} latent representations $\bm{M}^{(k)}_{c,t}$ for all samples $\mathcal{B}^{(k)}_c=\{(\bm{x}, y)\in \mathcal{B}^{(k)} \mid y=c \} $ belonging to class $c\in\mathcal{C}$. Given (class-specific) latent representations $\bm{M}_t^{(k)}$ and $\{\bm{M}_{c,t}^{(k)}\}_{c\in\mathcal{C}}$, we measure the low-dimensional structure and class separability using the Orthogonal Low-rank Embedding (OLE) criterion~\citep{ole}, which captures the degree of orthogonality between class-specific subspaces and is defined as: 
\[
\textstyle \text{OLE}^{(k)}_t \coloneqq \sum_{c\in\mathcal{C}} \| \bm{M}_{c,t}^{(k)} \|_* - \| \bm{M}_t^{(k)} \|_*,
\]
where $\| \cdot \|_*$ is the nuclear norm (sum of the singular values). %The nuclear norm serves as a convex relaxation of matrix rank, allowing us to evaluate the effective dimensionality of the subspaces. 
A lower OLE value indicates that the latent representations exhibit stronger intra-class low-rank structure and greater inter-class orthogonality, thereby reflecting a more structured and separable latent space. The OLE criterion has been widely used in prior work to enhance feature discriminability in classification tasks~\citep{ole,nuclear_cos}. In our setting, increasing OLE values suggest a trend toward a degenerate latent space, where class-specific subspaces lose structure and become less distinguishable.

\begin{figure}[t]
    \centering
    \includegraphics[width=.9\linewidth]{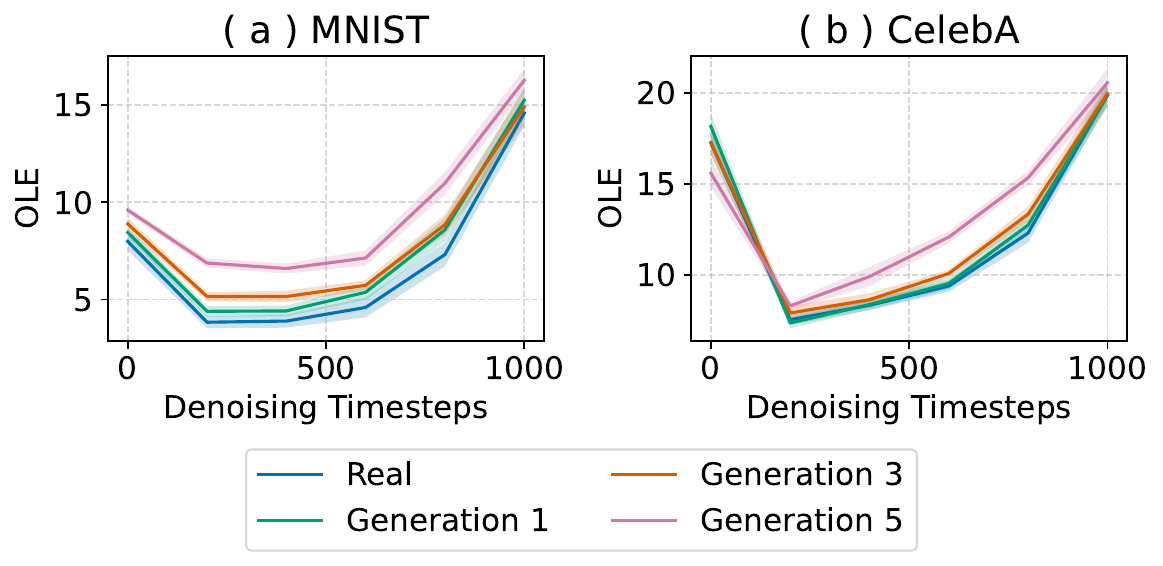}
    \caption{OLE values of latent representations extracted by a fixed diffusion model across generations and denoising timesteps. Each curve corresponds to a generation within a pure synthetic self-consuming loop. When conditioned on timestep, OLE values increase with generation, indicating progressive structural degradation of the latent space. When conditioned on generation, OLE exhibits a U-shaped trend over denoising timesteps.
    }
    \label{fig:datasets_ole}
\end{figure}

We first consider a pure synthetic self-consuming loop and train diffusion models for 6 generations, with models fine-tuned on 1,000 images for 3 epochs in each generation. For each image, we assign a label using sparse subspace clustering by orthogonal matching pursuit (SSC-OMP) algorithm~\citep{ssc-omp}. We compute OLE values for 10 batches of real and synthetic data. The results for MNIST \cite{mnist} and CelebA \cite{celeba} are shown in Figure~\ref{fig:datasets_ole}. Conditioned on generation, we observe that OLE initially decreases and then increases with the denoising timestep; this trend is consistent with recent findings that representation quality exhibits a unimodal trajectory~\citep{dynamic}. Lower OLE values are associated with greater feature separability, higher representation quality, and improved performance on downstream tasks. More importantly, when conditioned on the same timestep, OLE values increase as the self-consuming generation progresses. This suggests that the latent representations become less structured and increasingly entangled over time. 

The above observations are expected. Intuitively, latent features extracted by the diffusion model act as intermediaries that guide the generation of high-quality images. However, in self-consuming loops, synthetic data gradually deviates from the real data distribution. As a result, the learned low-dimensional structure diverges from the true data manifold, becoming less informative over time and ultimately leading to a decline in the quality of generated samples.

%We hypothesize that latent features act as intermediaries to guide the generation of high-quality images. In self-consuming loops, synthetic data often follows a distribution that deviates from that of real data. As a result, the learned low-dimensional structure gradually diverges from the intrinsic structure of the true data distribution and become less informative, ultimately leading to a decline in the quality of generated samples. 

% By filtering synthetic data which have low-dimensional structures that are less aligned with that of real images, it may be possible to prevent the learned distribution from drifting and thereby maintain high-quality generation across self-consuming iterations.

% These findings offer critical insights into model collapse. Prior work has proposed that diffusion models implicitly learn and sample from the low-dimensional structure of real image distributions, often modeled as mixtures of low-rank Gaussians. As the self-consuming process continues, the distribution of synthetic data diverges from the true data manifold. Consequently, the learned low-dimensional structure becomes less faithful to the intrinsic structure of real images, impairing the model's ability to generate realistic samples. Thus, the degradation of latent structure is not only a symptom but potentially a cause of generative collapse.

\noindent\textbf{Theoretical analysis.} The formulation of OLE is motivated by geometric intuition, aiming to quantify how well class-specific feature subspaces are aligned or separated. Prior work has shown that OLE is always non-negative and equals zero when the subspaces are perfectly orthogonal~\citep{ole}. 
However, the quantitative relationship between OLE and the degree of subspace orthogonality remains unexplored. 
In what follows, we address this gap by developing a theoretical framework that characterizes the connection between latent space degeneration and the dynamics of OLE. For clarity of presentation, we omit the diffusion timestep $t$ and generation index $k$ in the notations.

Given latent representation $\bm{h}_i$ of data sample $(\bm{x}_i,y_i)$ generated by encoder $e^{(0)}$ (Eqn.~\eqref{eq:represntaton}), we assume data are from two classes $\mathcal{C}=\{0,1\}$ and  latent representations $\bm{h}_i\in\mathbb{R}^{d\times 1}$ for each class $c\in \{0,1\}$ are drawn from noisy low-rank Gaussian distributions. Mathematically, we model conditional distribution of $\bm{h}_i\mid y_i=c$ as zero-mean Gaussian  $\mathcal{N}(\bm{0}, \bm{\Sigma}_c)$ with covariance  $\bm{\Sigma}_c = \bm{U}_c \bm{U}_c^\top + \sigma^2 \bm{I}_d$. Here, $\bm{U}_c \in \mathbb{R}^{d \times r}$ is an orthonormal basis matrix representing the low-dimensional subspace of class $c$ with rank $r < d$, and $\sigma^2 \bm{I}_d$ represents isotropic noise.

Let matrix $\bm{M}_c \in \mathbb{R}^{d \times n}$ be the \textit{class-specific} latent representations for $n$ samples from class $c$, with each column the representation of one sample. 
Let $\bm{M} = [\bm{M}_0 \;\; \bm{M}_1] \in \mathbb{R}^{d \times 2n}$ denote the concatenated matrix of a batch of samples from two classes. The OLE for this sample matrix becomes:
\[
\text{OLE}(\bm{M}_0, \bm{M}_1) = \|\bm{M}_0\|_* + \|\bm{M}_1\|_* - \|\bm{M}\|_*,
\]
where $\|\cdot\|_*$ denotes the nuclear norm. For any column unit vectors $\bm{u}_i$ of $\bm{U_0}$ and $\bm{v}_j$ in $\bm{U}_1$, denote their angle as $\theta_{ij}$, i.e., $\bm{u}_i^\top \bm{v}_j = \cos \theta_{ij}$, $i,j \in [r].$ Suppose that the cosine similarity of each pair of unit basis can be bounded by:
\[
\cos \tilde{\theta} \leq \cos \theta_{ij} \leq \ell \cos \tilde{\theta},
\]
where $\ell$ measures the range of pairwise cosine similarities, $\tilde{\theta} \in [0, \frac{\pi}{2}]$ is the largest pairwise angle.
Theorem \ref{theorem:ole_lob} below provides a lower bound on the expected value of OLE, characterized by the degree of orthogonality between subspaces.

\begin{theorem}[Lower bound of OLE]\label{theorem:ole_lob}
For all $\theta\in [0,\frac{\pi}{2}]$, when $r > 2n$, the expected OLE satisfies the following:
    \[
    \mathbb{E}[\text{OLE}(\bm{M}_0, \bm{M}_1)] \geq C_1 - C_2 \cdot \phi(\tilde{\theta}),
    \]
    where
    \begin{align*}
    \phi(\tilde{\theta}) & \coloneqq \sqrt{2n} \cdot \sqrt{{\ell \cos\tilde{\theta}}} + \sqrt{2n(2n - 1)} \cdot \sqrt{1 - \cos\tilde{\theta}} \; , \\
      C_1 &\coloneqq 2\gamma(rn)
      - 2\sqrt{n} \gamma(dn)
      - \sqrt{2n} \cdot \gamma(2dn), \\
    C_2 &\coloneqq 2 \gamma(rn),
    \end{align*}
    and $\gamma(n) \coloneqq \sqrt{2} \cdot \frac{ \Gamma\left( \frac{n+1}{2} \right) }{ \Gamma\left( \frac{n}{2} \right) }$, with $\Gamma(\cdot)$ being Gamma function. 
\end{theorem}
The proof of Theorem~\ref{theorem:ole_lob} is provided in Appendix~\ref{proof:ole_lob}.
This result offers a theoretical explanation for the empirical trend observed in Figure~\ref{fig:datasets_ole}, where OLE values increase across successive generations. More specifically, assuming that as the self-consuming process progresses, the quality of the generated data decreases and the learned latent subspaces become less orthogonal (i.e., $\theta$ between subspaces $\text{span}(\bm{U}_0)$ and $\text{span}(\bm{U}_1)$ decreases), Theorem~\ref{theorem:ole_lob} shows that as $\theta$ decreases, the lower bound on the expected OLE increases.

\begin{remark}
Since $\phi(\theta)$ in Theorem~\ref{theorem:ole_lob}  is decreasing as $\theta$ decreases over $[0, \arccos \frac{1}{2n}]$, the lower bound of OLE increases as the subspaces $\bm{U}_0$, $\bm{U}_1$ become less orthogonal. When $n \to \infty$, we have $\arccos \frac{1}{2n} \to \frac{\pi}{2}$, indicating that the conclusion in Theorem~\ref{theorem:ole_lob} holds for most $\theta\in[0,\frac{\pi}{2}]$.    
\end{remark}

%This establishes a connection between deteriorating latent representations and increasing OLE values. %, which, where OLE values rise as the self-consuming generation progresses.    

%In the setting of Theorem~\ref{theorem:ole_lob}, this corresponds to a smaller angle $\theta$ between subspaces $\text{span}(\bm{U}_0)$ and $\text{span}(\bm{U}_1)$. As $\theta$ decreases over $[0, \frac{\pi}{4}]$, the term $\phi(\theta)$ also decreases, causing the lower bound on the expected OLE to increase. This theorem supports the connection between worsening latent representations and increasing OLE.

\noindent\textbf{Mitigate model collapse through filtering.}
Previous studies have suggested that removing low-quality synthetic samples from the training data can be more effective in stabilizing retraining than adding high-quality data~\citep{remove}. However, identifying synthetic images directly within a mixture of real and synthetic data is challenging~\citep{synthetic_property,verification}. The observations in Figure~\ref{fig:datasets_ole} suggest a potential direction:
%raise a natural question: 
\begin{center}
 \textit{Can synthetic images be identified and filtered based on the quality of their latent representations?}   
\end{center}

\begin{figure}[t]
  \begin{minipage}[t]{.23\textwidth}
    \includegraphics[width=.85\textwidth]{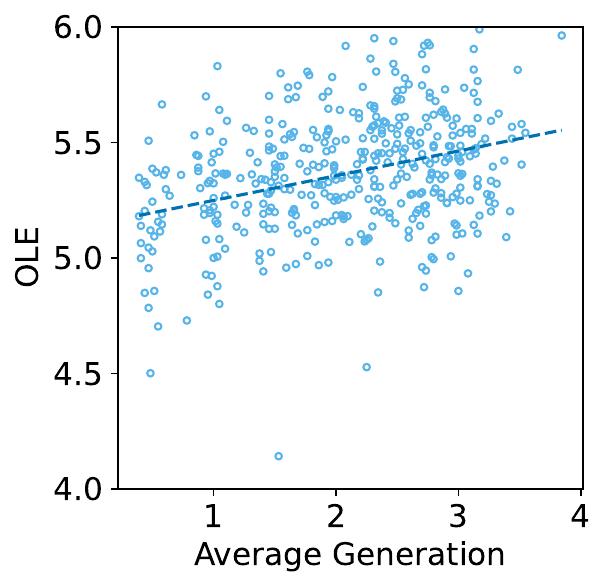}
    \caption{Correlation between the average generation number of each batch and its corresponding OLE score, computed on the accumulated CelebA dataset. Batches with lower average generation numbers (i.e., containing more realistic images) tend to have lower OLE values.
    %, indicating better low-dimensional structure. 
    However, the correlation is not strong enough to enable reliable filtering based solely on batch-level OLE.}
    \label{fig:celeba_ole_reg}
  \end{minipage}
  \hfill
  \begin{minipage}[t]{.23\textwidth}
    \includegraphics[width=.85\textwidth]{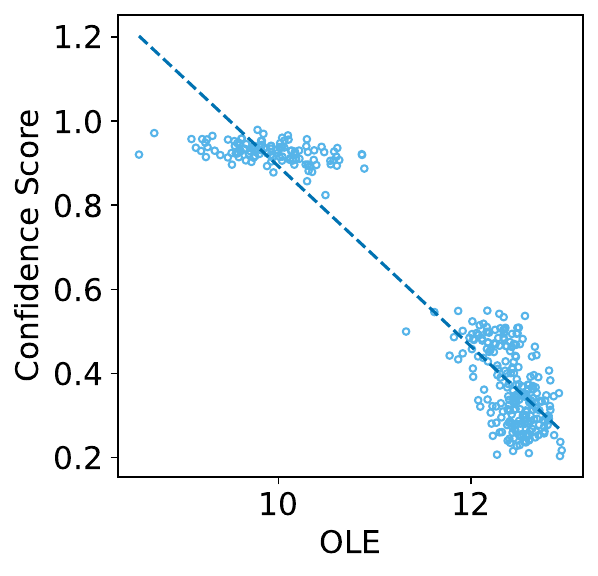}
    \caption{Correlation between the OLE and average confidence scores of each batch, computed on the real and synthetic MNIST datasets produced at different generations. Batches with lower OLE (i.e., greater feature separability) tend to have higher confidence scores, indicating that we can use confidence score as an individual-level proxy of batch-level OLE score.}
    \label{fig:mnist_ole_acc}
  \end{minipage}
\end{figure}

A straightforward idea is to compute the OLE score for each batch $\mathcal{B}^{(k)}$ and discard those whose scores exceed a certain threshold. While this approach is intuitive, it becomes less effective in realistic scenarios where real and synthetic images are mixed within each batch. Unlike in our controlled experiments where each batch contained either only real or only synthetic data (as shown in Figure~\ref{fig:datasets_ole}), mixed batches dilute the signal that OLE captures. 

To assess the feasibility of OLE-based filtering in such mixed scenarios, we combine the accumulated datasets $\mathcal{D}^{(0)}, \ldots, \mathcal{D}^{(5)}$ from the Accumulation Loop and compute the OLE score for each batch. Each image is tagged with the generation index from which it was sampled, allowing us to calculate the average generation number per batch. A lower average generation number indicates that the batch contains more samples from earlier stages of self-consuming training, which are less affected by model collapse and therefore more likely to be of higher quality. As shown in Figure~\ref{fig:celeba_ole_reg}, there is a positive correlation between batch OLE scores and the average generation number. Since earlier generations contain more realistic samples, this correlation supports the hypothesis that lower OLE scores reflect higher-quality latent representations. However, the correlation is relatively weak, and filtering data batches based on OLE scores is not a reliable strategy—particularly when each batch contains a mix of real and synthetic images. Indeed, as shown in Figure~\ref{fig:mnist_clf_filter}, OLE-based filtering does not yield significant improvement over random sampling in stabilizing self-consuming training. We thus introduce an alternative proxy of OLE scores with a better filtering algorithm, as detailed in Section~\ref{sec:method}.

\section{Proposed method: latent space filtering}
\label{sec:method}
Section~\ref{sec:observation} shows that OLE scores computed over data batches cannot serve as a reliable criterion for filtering synthetic data, particularly when each batch contains a mix of real and synthetic samples. We therefore introduce an alternative representation-based filtering mechanism, which is related to OLE but more effective, as detailed below.

\smallskip
\noindent \textbf{Motivation.}
As shown in Figure~\ref{fig:datasets_ole}, OLE exhibits a U-shaped trend over diffusion timestep $t$ when conditioned on generation number $k$, and increases with generation $k$ when conditioned on timestep $t$. Interestingly, a similar U-shaped pattern was observed in~\citep{dynamic}, which demonstrates that the \textit{representation quality} of diffusion models during the denoising process follows a unimodal trajectory.  We thus hypothesize that representation quality---measured by the accuracy of a probing classifier (a simple, often linear classifier applied to latent representations to assess their informativeness)---also reflects the underlying low-dimensional structure, similar to what is captured by OLE scores. 

Following this idea, we propose a \textit{representation-based} filtering approach. Assuming that a probing classifier trained on the latent representations of \textit{real} images captures the underlying low-dimensional structure, we can then use the classifier's confidence scores to measure the degree of misalignment between a sample's representation and those of real data. Specifically, we first extract latent representations of real images using a well-trained diffusion model, then train a softmax regression model based on these representations. Let the output logit for class $c$ be defined as ${o}_c(\bm{x}) \coloneqq \bm{w_c}^\top\bm{x} + \bm{b}_c$. Then, for an image $\bm{x}$ with label $y$, we compute its \textbf{confidence score}  regarding the correct class by softmax function:
\begin{equation}\label{eq:confidence}
    \xi(\bm{x}, y) \coloneqq \frac{\exp \{o_y(\bm{x})\}}{\sum_{c\in\mathcal{C}} \exp \{o_c(\bm{x})\}}.
\end{equation}
We hypothesize that images with higher confidence scores are more likely to resemble real data and can therefore serve as a reliable filtering criterion. To verify this hypothesis, we conduct experiments on the MNIST dataset, with results shown in Figure~\ref{fig:mnist_clf_filter}.

%shows the U-shape trend of OLE conditioning on the generation, and decreasing of OLE conditioning on the input timestep. However, filtering with batch-wise OLE is too coarse-grained for removing synthetic images. Recent works have shown a similar unimodal trend in representation quality of diffusion models in the denoising process, and uses probing accuracy, i.e, accuracy of applying linear classifier to representations, to measure representation quality~\citep{dynamic}. Motivated by this, we assume the probing classifier trained on representations of real images have captured the low-dimensional structure, and therefore uses the confidence score of probing classifier to measure the misalignment of the evaluated representations compared with real representations. Specifically, we first extract the representations of real images on a well-trained diffusion model, then train a softmax regression model based on these representations. The output scalar of each class $c$ is denoted as ${o}_c(\bm{x}) \coloneqq \bm{w_c x} + \bm{b}_c$. Then for image $\bm{x}$ with ground truth label $y$, we can compute its confidence score regarding the correct class by softmax function $\xi(\bm{x}, y) \coloneqq \frac{\exp o_y{\bm{x}}}{\sum_c \exp o_c(\bm{x})}$.

\begin{figure*}[htbp]
    \centering
    \begin{subfigure}{0.28\textwidth}
        \centering
        \includegraphics[width=.75\linewidth]{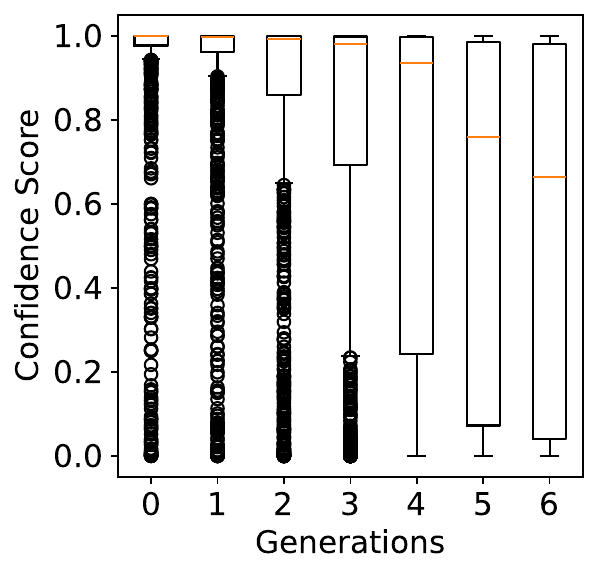}
        \caption{Generations vs. Confidence scores}
        \label{fig:mnist_clf}
    \end{subfigure}
    \hfill
    \begin{subfigure}{0.27\textwidth}
        \centering
        \includegraphics[width=.75\linewidth]{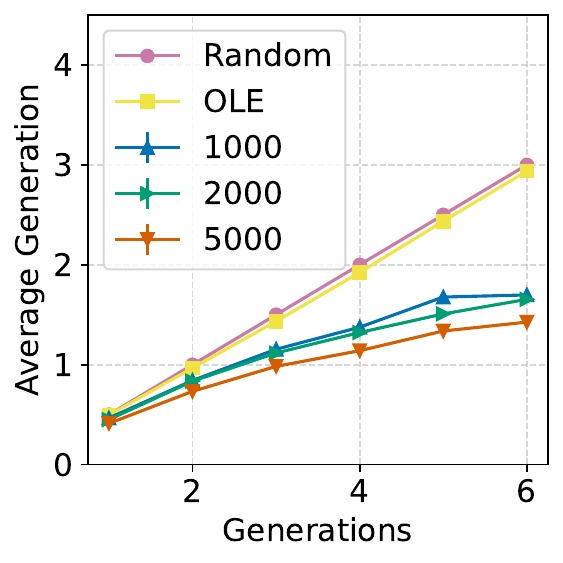}
        \caption{Average generation after filtering}
        \label{fig:mnist_clf_fiter_gen}
    \end{subfigure}
    \hfill
    \begin{subfigure}{0.27\textwidth}
        \centering
        \includegraphics[width=.75\linewidth]{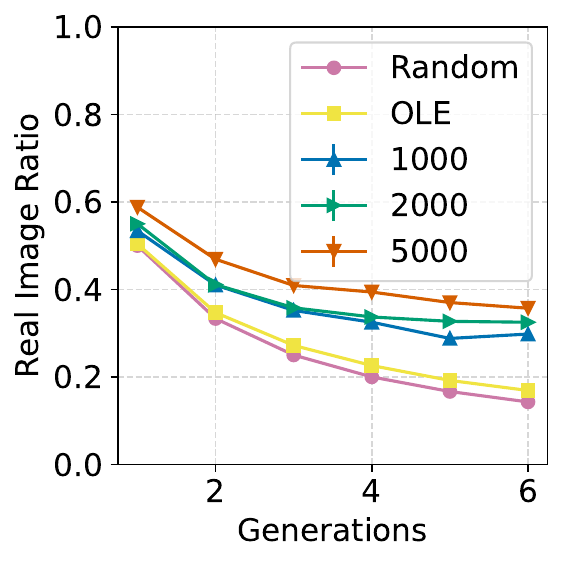}
        \caption{Ratio of real images after filtering}
        \label{fig:mnist_clf_filter_ratio}
    \end{subfigure}
    \caption{The effect of using confidence scores for filtering on MNIST dataset. (a) shows the generation number vs. distribution of confidence scores of images sampled at this generation. Higher generation is correlated with lower confidence score, indicating that they are less aligned with real images, hence providing a strong signal for filtering out unrealistic synthetic images. (b) shows the average generation number within the dataset filtered by the confidence score. (c) Our method can also preserve more real images in the training set. Using larger accumulated dataset for filtering can increase the ratio of selected real images. Compared with randomly sampling from the accumulated dataset, our filtering method can preserve a distribution closer to real images.}
    \label{fig:mnist_clf_filter}
\end{figure*}
Figure~\ref{fig:mnist_clf} presents box-and-whisker plots of confidence scores across different generations\footnote{Each box spans from the first quartile (Q1) to the third quartile (Q3) of the data, with whiskers extending to 1.5 times the interquartile range (IQR). Data points beyond the whiskers are shown as fliers, indicating outliers. %while the whiskers extend to 1.5 times the interquartile range (IQR) from the box. Points beyond the whiskers are plotted as fliers, representing outliers.
}.
The results reveal a clear correlation: \textbf{images from later generations tend to have lower confidence scores}. This suggests that confidence scores can serve as a more fine-grained filtering criterion compared to batch-level OLE scores. Furthermore,  Figure~\ref{fig:mnist_ole_acc} shows a \textbf{positive correlation between low OLE scores and high average confidence scores across batches}, suggesting that confidence scores can serve as a viable proxy for OLE scores. Unlike OLE, which operates at the batch level, confidence scores are computed per image, allowing us to circumvent the previously mentioned issue of mixed real and synthetic samples within the same batch.

To evaluate the effectiveness of confidence-based filtering, we compare it with random sampling in Figures~\ref{fig:mnist_clf_fiter_gen} and~\ref{fig:mnist_clf_filter_ratio}. For each generation number $k$, we construct an accumulated dataset by concatenating the real dataset with a series of synthetic datasets $\widehat{\mathcal{D}}^{(1)}, \ldots, \widehat{\mathcal{D}}^{(k)}$, where the size of these datasets increases from 1,000 to 5,000 images, as indicated in the figures. Using a pre-trained probing classifier, we score all images and select the top 1,000 with the highest confidence. Figure~\ref{fig:mnist_clf_fiter_gen} reports the average generation number of the filtered images: lower values indicate samples that are closer to the real distribution. Since real images play a more critical role in mitigating model collapse in self-consuming loops~\citep{mad,recrusive_stability,self_consuming_diffusion}, Figure~\ref{fig:mnist_clf_filter_ratio} further reports the proportion of real images in the filtered set. The results demonstrate that \textbf{filtering using the probing classifier yields both a lower average generation number and a higher fraction of real images} compared to random sampling. Furthermore, as the size of the accumulated dataset increases, the filtering becomes more effective because the larger pool of real images results in the filtered fixed-size dataset containing more real images.

\begin{algorithm}[tb]
\caption{Latent Space Filtering}
\label{alg}
\textbf{Input:} Pre-trained diffusion model $f^{(0)}$, training budget $N$, self-consuming loop $K$.
\begin{algorithmic}[1]
\STATE Train a probing classifier with $f^{(0)}$ and real images $\mathcal{{D}}^{(0)}$
\FOR{$k = 1$ to $K$}
    \STATE Sample a new dataset $\mathcal{\widehat{D}}^{(k)}$ from model ${f}^{(k-1)}$.
    \STATE Construct a dataset $\mathcal{D}^{(k)}$ according to the type of the self-consuming loop.
    \STATE Compute the confidence score $\xi(\bm{x_i}, y_i)$ for each image $(\bm{x_i}, y_i) \in \mathcal{D}^{(k)}$ by Eqn.~\eqref{eq:confidence}.
    \STATE Select top-$N$ images with highest scores to construct a new set $\underline{\mathcal{D}}^{(k)}$ with training budget $N$.
    \STATE Train a new diffusion model ${f}^{(k)}$ using dataset $\underline{\mathcal{D}}^{(k)}$.
\ENDFOR
\STATE \textbf{Output:} $f^{(K)}$
\end{algorithmic}
\end{algorithm}

In Algorithm~\ref{alg}, we summarize the complete procedure for applying the proposed filtering method to self-consuming accumulation training loop. It is worth noting that access to real images is only required at the start of the self-consuming loops. Afterward, there is no need to distinguish real images from the synthetic dataset in the accumulated dataset. Additionally, even when images are unlabeled, labeling or captioning each image can be easier than detecting synthetic images or acquiring new data~\citep{synthetic_caption, human_perception}. By retaining as many real and realistic synthetic images as possible in the filtered dataset, we aim to prevent model collapse while adhering to a fixed training budget.

\smallskip
\noindent \textbf{Theoretical analysis.}  Following the same setting as in Theorem~\ref{theorem:ole_lob}, we next theoretically characterize the relation between confidence scores and lower-dimensional structure of latent representations. For samples from two classes $\mathcal{C}=\{0,1\}$, let the log-likelihood ratio for a sample representation $\bm{h}$ be defined as
$g(\bm{h}) = \log \frac{\Pr[\bm{h} \mid y = 1]}{\Pr[\bm{h} \mid y = 0]}$. Then,
the posterior probability $\Pr[y = 1 \mid \bm{h}]$ assigned by the Bayes optimal classifier is given by $\varsigma(g(\bm{h}))$, where $\varsigma(\cdot)$ denotes the Sigmoid function. We define the expected confidence score for class 1, under the distribution of class-1 samples, as:
\[
\xi(\theta) := \mathbb{E}_{\bm{h}|y=1 \sim \mathcal{N}(\bm{0}, \bm{\Sigma}_1)}[\varsigma(g(\bm{h}))],
\]
where $\theta \in [0, \frac{\pi}{2}]$ is the angle between subspaces spanned by $\bm{U}_0$ and $\bm{U}_1$. Similar to Theorem~\ref{theorem:ole_lob}, the following theorem characterizes the relationship between the confidence score and the degree of orthogonality between subspaces. 
%We then establish the following upper bound:

\begin{theorem}[Upper bound of confidence score]\label{theorem:score}
The expected confidence score satisfies the following:
\[
\xi(\theta) \leq \frac{1}{2\sigma^2 (\sigma^2+1)} \varsigma(r \sin^2 \theta),
\]
\end{theorem}
where $r$ is the rank of $\bm{U}_0$ and $\bm{U}_1$.
The proof is provided in Appendix~\ref{proof:score}. Note that the upper bound is monotonically increasing in $\theta$. As $\theta \to 0$, the subspaces $\text{span}(\bm{U}_0)$ and $\text{span}(\bm{U}_1)$ become more aligned, making the two classes less distinguishable. Consequently, the upper bound of the expected confidence score decreases, reflecting the reduced separability between classes.

\begin{figure*}[!htbp]
    \centering
    \includegraphics[width=.65\linewidth]{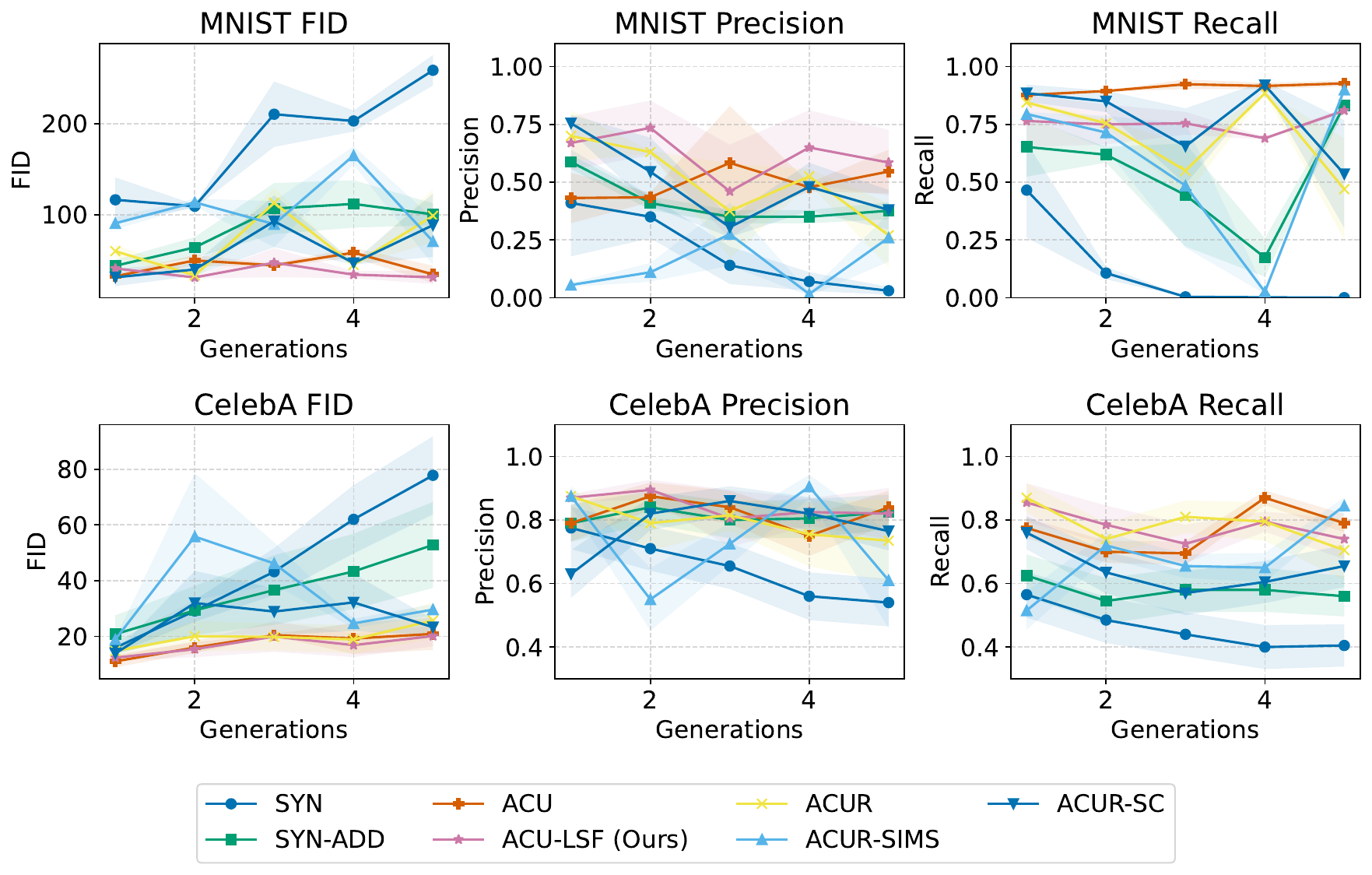}
    \caption{Performance comparison on MNIST and CelebA across three metrics: FID, precision, and recall. FID measures distributional distance between real and synthetic images (lower is better), while precision and recall assess fidelity and diversity of generated samples, respectively (higher is better). SYN suffers from model collapse. SYN-ADD partially alleviates it via fresh real data. ACU, ACUR, and ACU-LSF maintain stable metrics across generations. ACU-LSF achieves lower FID and higher fidelity than ACUR. ACUR-SIMS exhibits instability, while ACUR-SC suffers from reduced recall and elevated FID.}
    \label{fig:celeba_baselines}
\end{figure*}
\section{Experiments}
\label{sec:exp}

\underline{\textbf{Datasets.}} We evaluate our approach by training DDPM diffusion models~\citep{ddpm} for conditional image synthesis~\cite{beat_gan} on three real-world datasets: %\begin{enumerate}[leftmargin=*]
%    \item 
(i) \textbf{MNIST}~\citep{mnist}: A dataset of 60,000 grayscale handwritten digits across 10 classes, with image size $28 \times 28$.
 %   \item 
 (ii) \textbf{CIFAR-10}~\citep{cifar10}: A collection of 60,000 $32 \times 32$ RGB images spanning 10 object categories.
   % \item 
   (iii) \textbf{CelebA}~\citep{celeba}: A dataset of 200,000 celebrity face images, each annotated with 40 binary attributes\footnote{We downsample each image to $64 \times 64$ and define 4 classes based on the intersection of two attributes: “Male” and “Smile.”}.
%\end{enumerate}

\noindent\underline{\textbf{Baselines.}} For each dataset, we first train a baseline model on the full dataset. We then iteratively fine-tune the model for five self-consuming generations.
Our method, \textit{ACU-LSF}, is compared with the following five baselines:

\begin{enumerate}[leftmargin=*,topsep=0cm,itemsep=0cm]
    \item \textbf{SYN}: Each generation is trained purely on synthetic data generated by the previous model.
    
    \item \textbf{SYN-ADD}~\citep{mad}: A hybrid setting where each generation's training data consists of 70\% synthetic images and 30\% newly added real images. The 30\% proportion is chosen to approximate the average percentage of real images present in the accumulated datasets over five generations.
    
    \item \textbf{ACU}~\citep{accumulate}: Historical real and synthetic data are accumulated across generations, resulting in a linearly growing training set. A subset of real data is sampled for accumulation in each generation, matching the size of newly generated synthetic data.
    
    \item \textbf{ACUR}~\citep{remove}: A fixed-size subset is randomly sampled from the accumulated dataset used in ACU. This setting evaluates the impact of maintaining a fixed training budget.
    
    \item \textbf{ACUR-SIMS}~\citep{sims}: This method extrapolates the score function to enhance fidelity. At generation $k$, the score function used for sampling is:
    \[
    {\bm{s}}^{(k)}_{\theta}(\bm{x}_t, t) = (1+\omega) \bm{s}^{(k-1)}_{\theta}(\bm{x}_t, t) - \omega \widehat{\bm{s}}^{(k)}_{\theta}(\bm{x}_t, t),
    \]
    where $\widehat{\bm{s}}^{(k)}_{\theta}(\cdot)$ corresponds to the model trained with the dataset constructed in the same way as ACUR, $\omega$ is a hyperparameter controlling the strength of extrapolation. The synthetic dataset $\widehat{\mathcal{D}}^{(k)}$ is sampled using this new score function ${\bm{s}}^{(k)}_{\theta}(\cdot)$.
    
    \item \textbf{ACUR-SC}~\citep{self_correction}: This method clusters real images to obtain class centers and maps each synthetic image to the closest center during training. Besides, the training set is constructed in the same way as ACUR.
\end{enumerate}

\noindent\underline{\textbf{Evaluation metrics.}} Each generation uses 1,000 samples for fine-tuning (except ACU, which accumulates an additional 1,000 images per generation). Training is conducted for three epochs per generation. After each generation, 10,000 images are sampled to evaluate three standard metrics: %\begin{enumerate}[leftmargin=*,itemsep=0cm,topsep=0cm]\item 
    (i) \textbf{Fr\'{e}chet Inception Distance (FID)}~\citep{fid} which  quantifies the distance between real and synthetic image distributions. (ii) \textbf{Precision}~\citep{improved_metric} which measures how many synthetic samples fall on the real data manifold (fidelity).
(iii) \textbf{Recall}~\citep{improved_metric} which measures how many real samples fall on the synthetic manifold (diversity).
%\end{enumerate}
Note that lower FID and higher precision and recall values indicate better generative performance. Samples of synthetic images generated by self-consuming models are provided in Fig.~\ref{fig:cifar10_grid} (Appendix~\ref{appendix:exp}) to show the model quality.

\noindent\underline{\textbf{Results.}} We primarily present results for MNIST and CelebA in Figure~\ref{fig:celeba_baselines}, while additional results for CIFAR-10 are provided in Appendix~\ref{appendix:exp}.
As expected, performance under \textbf{SYN} deteriorates over generations due to model collapse. \textbf{SYN-ADD} partially mitigates collapse by introducing fresh real data. Prior work has shown that increasing the real-to-synthetic ratio improves performance~\citep{remove}, but collecting fresh real data may not always be feasible. Our goal is to demonstrate that effective filtering without requiring additional real data can offer comparable benefits.
\textbf{ACU} maintains stable performance across generations, in line with previous observations~\citep{accumulate}. By randomly sampling from the accumulated dataset, \textbf{ACUR} approximates the performance of ACU while adhering to a fixed training budget. As demonstrated in Figure~\ref{fig:mnist_clf_filter}, our method \textbf{ACU-LSF} successfully removes recent synthetic images and retains more real images, resulting in a distribution that more closely aligns with the original real data. Consequently, ACU-LSF achieves comparable precision and recall, and a lower FID than other fixed-budget baselines. Besides, ACU retrains a diffusion model with 32 M parameters on 582 MB of CelebA data (5 K images at generation 5), and the cost increases with training budget and generations. In contrast, ACU-LSF trains a 65 K-parameter probing classifier on 109 MB of features once at the start, greatly reducing storage and computation.

In contrast, \textbf{ACUR-SIMS} and \textbf{ACUR-SC} perform less favorably, even when combined with accumulation. SIMS~\citep{sims} was designed to improve fidelity within a single generation via extrapolation. However, over multiple generations, its score functions become unstable due to repeated extrapolation, potentially drifting further from the real data distribution. As a result, metrics fluctuate and performance deteriorates.
ACUR-SC employs a self-correction mechanism by aligning synthetic samples with cluster centers of real images. While this exploits structural information from real data, it reduces sample diversity, leading to lower recall and higher FID despite moderate precision.

Overall, our results demonstrate that LSF based on latent representation quality is an effective alternative to adding real data or accumulation, offering a scalable and practical solution to mitigate model collapse in self-consuming training.

% \section{Conclusions}

% As generative models increasingly rely on mixtures of human- and model-generated data, concerns arise regarding potential degradation in model quality and diversity. In this work, we address the phenomenon of model collapse caused by self-consuming training loops and propose a filtering algorithm to remove unrealistic synthetic images, thereby stabilizing the retraining process.

% Our key insight is that the low-dimensional structure of latent representations extracted by diffusion models deteriorates over successive generations. We leverage this observation to assess the misalignment between collected training samples and the underlying real data distribution. By filtering out samples with poor alignment, we are able to improve stability during retraining.
% We introduce a unified theoretical framework to support our observations and approach. We validate the effectiveness of our approach through empirical results on real-world datasets. Under a fixed training budget, our method  outperforms competitive baselines in image quality, diversity, and stability.

\section*{Acknowledgements}
This work was funded in part by the National Science Foundation under award number IIS-2202699,
IIS-2416895, and IIS-2416896.

\bibliography{ref}

% \clearpage
% \input{ReproducibilityChecklist.tex}

\clearpage
\appendix
\setcounter{secnumdepth}{1}

\section{Proof of Theorem~\ref{theorem:ole_lob}}
\label{proof:ole_lob}

We first propose the following lemmas that will be used later.

\begin{lemma}
\label{lemma:nuclear_cos}
\textit{For matrix $\bm{V} \in \mathbb{R}^{d \times N}$ which contains $N$ $d$-dimensional unit vectors $\{\bm{v_i}\}$, $d > N$, the nuclear norm of $\bm{V}$ is bounded by:}
\[
\|\bm{V}\|_* 
\leq \sqrt{N} \cdot \sqrt{\overline{\cos\theta}} + \sqrt{N(N - 1)} \cdot \sqrt{1 - \overline{\cos\theta}},
\]
where
\[
\overline{\cos\theta} = \frac{1}{N^2} \sum_{i=1}^N \sum_{j=1}^N \cos\theta_{ij} = \frac{1}{N^2} \sum_{i=1}^N \sum_{j=1}^N \bm{v_i}^\top \bm{v_j}
\]
is the average cosine similarity between unit vectors.
\end{lemma}

Lemma~\ref{lemma:nuclear_cos} is proposed in~\cite{nuclear_cos} as Theorem 4.1.

\begin{lemma}
\label{lemma:nuclear_triangle}
    \textbf{(Triangle inequality of nuclear norm)} $\|\bm{A} + \bm{B}\|_* \leq \|\bm{A}\|_* + \|\bm{B}\|_*$.
\end{lemma}

\begin{lemma}\label{lemma:nuclear_fnorm}
    $\|\bm{X}\|_F \leq \|\bm{X}\|_* \leq \sqrt{r}\|\bm{X}\|_F$, where $r$ is the rank of $\bm{X}$.
    %~\citep{fnorm_bound}.
\end{lemma}

\begin{lemma}\label{lemma:nuclear_opnorm}
    $\|\bm{AB}\|_F \leq \|\bm{A}\|_{op} \|\bm{B}\|_{*} $, or symmetrically, $\|\bm{AB}\|_F \leq \|\bm{A}\|_{*} \|\bm{B}\|_{op} $, where $\|\cdot\|_{op}$ is the operator norm, i.e, the largest singular value.
\end{lemma}

Lemma~\ref{lemma:nuclear_opnorm} is an instance of the H\"older's inequality for Schatten norms.

% \begin{lemma}\label{lemma:nuclear_fnorm_mul}
%     \textbf{(Upper bound of nuclear norm)} $\|\bm{UV}\|_* \leq \|\bm{U}\|_F \cdot \|\bm{V}\|_F$.
% \end{lemma}

% \begin{lemma}\label{lemma:nuclear_orthonormal}
%     \textbf{(Invariance under orthonormal transformation)} If $\bm{U}$ is an orthonormal matrix, then $\|\bm{U}\bm{V}\|_* = \|\bm{V}\|_*$.
% \end{lemma}

\begin{lemma}\label{lemma:op_fnorm}
    $\|\bm{X}\|_{op} \leq \|\bm{X}\|_F$.
\end{lemma}

\begin{lemma}\label{lemma:exp_fnorm}
    For a matrix $\bm{X} \in \mathbb{R}^{m \times n}$, with each entry $\bm{X_{ij}} \overset{\text{i.i.d.}}{\sim} \mathcal{N}(0,1)$, then $\mathbb{E}(\|\bm{X}\|_F) = \sqrt{2} \cdot \frac{ \Gamma\left( \frac{mn+1}{2} \right) }{ \Gamma\left( \frac{mn}{2} \right) }$, where $\Gamma(\cdot)$ is the Gamma function.
\end{lemma}

\begin{proof}
    The Frobenius norm of the specified matrix follows Chi distribution with $mn$ degrees of freedom, i.e.,  $\|X\|_F = \sqrt{\sum_{i=1}^m \sum_{j=1}^n \bm{X}_{ij}^2} \sim \chi_{mn}$. Hence $\mathbb{E}(\|X\|_F) = \sqrt{2} \cdot \frac{ \Gamma\left( \frac{mn+1}{2} \right) }{ \Gamma\left( \frac{mn}{2} \right) }$
\end{proof}

We now describe the notations used in our theorem. We assume latent representations $\bm{h}_i\in\mathbb{R}^{d\times 1}$ for each class $c\in \{0,1\}$ are drawn from the union of two noisy low-rank Gaussian distributions. We model conditional distribution of $\bm{h}_i\mid y_i=c$ as $\mathcal{N}(\bm{0}, \bm{\Sigma}_c)$ with covariance  $\bm{\Sigma}_c = \bm{U}_c \bm{U}_c^\top + \sigma^2 \bm{I}_d$. Here, $\bm{U}_c \in \mathbb{R}^{d \times r}$ is an orthonormal basis matrix representing the low-dimensional subspace of class $c$ with rank $r < d$, and $\sigma^2 \bm{I}_d$ represents isotropic noise. Hence we can denote $\bm{h}_i = \bm{U}_c \bm{z} + \bm{\varepsilon}$, with $\bm{z} \sim \mathcal{N}(0, \bm{I}_r)$ and noise $\bm{\varepsilon} \sim \mathcal{N}(0, \sigma^2 \bm{I}_d)$.
Let matrix $\bm{M}_c \in \mathbb{R}^{d \times n}$ be the class-specific latent representations for $n$ samples from class $c$, with each column the representation of one sample. 
Let $\bm{M} = [\bm{M}_0 \;\; \bm{M}_1] \in \mathbb{R}^{d \times 2n}$ denote the concatenated matrix of samples from both classes. The OLE for this sample matrix becomes:
\[
\text{OLE}(\bm{M}_0, \bm{M}_1) = \|\bm{M}_0\|_* + \|\bm{M}_1\|_* - \|\bm{M}\|_*,
\]
where $\|\cdot\|_*$ denotes the nuclear norm.
For any column unit vectors $\bm{u}_i$ of $\bm{U_0}$ and $\bm{v}_j$ in $\bm{U}_1$, denote their angle as $\theta_{ij}$, i.e., $\bm{u}_i^\top \bm{v}_j = \cos \theta_{ij}$, $i,j \in [r].$ Suppose that the cosine similarity of each pair of unit basis can be bounded by:
\[
\cos \tilde{\theta} \leq \cos \theta_{ij} \leq \ell \cos \tilde{\theta},
\]
where $\ell$ measures the range of pairwise cosine similarities, $\tilde{\theta} \in [0, \frac{\pi}{2}]$ is the largest pairwise angle.
Now we can prove Theorem~\ref{theorem:ole_lob}.

\begin{proof}

Denote $\bm{M_c} =\bm{U_c} \bm{Z_c} + \bm{E_c}$, where $\bm{Z_c} \in \mathbb{R}^{r \times n}$ is the matrix of coordinate terms, $\bm{E_c} \in \mathbb{R}^{d \times n}$ is the matrix of noises terms. 
We first lower bound class-specific matrix $\|\bm{M}_c\|_*$ by:

\begin{align*}
    \|\bm{M}_c\|_* & \geq \|\bm{U}_c \bm{Z}_c\|_* - 
    \|-\bm{E}_c\|_* 
    \quad (\text{Lemma}~\ref{lemma:nuclear_triangle})
    \\
    & \geq \|\bm{U}_c \bm{Z}_c\|_F
    - \sqrt{n}\|\bm{E}_c\|_F
    \quad (\text{Lemma}~\ref{lemma:nuclear_fnorm}) \\
    & = \|\bm{Z}_c\|_F - \sqrt{n}\|\bm{E}_c\|_F
\end{align*}

Then we upper bound $\bm{M}$ by the following:

\begin{align*}
\| \bm{M} \|_* &= \|[\bm{U}_0 \bm{Z}_0 + \bm{E}_0 \quad \bm{U}_1 \bm{Z}_1 + \bm{E}_1] \|_* \\
&\leq \|\underbrace{[\bm{U}_0 \;\; \bm{U}_1]}_{\bm{\widetilde{U}}}
\underbrace{\begin{bmatrix}
    \bm{Z}_0 \\
    \bm{Z}_1
\end{bmatrix}}_{\bm{\widetilde{Z}}} \|_*
+ \| \underbrace{[\bm{E}_0 \;\; \bm{E}_1]}_{\bm{\widetilde{E}}} \|_* 
\quad (\text{Lemma}~\ref{lemma:nuclear_triangle})
\\
&\leq \|\bm{\widetilde{Z}}\|_{op} \cdot \|\bm{\widetilde{U}} \|_*
+ \sqrt{2n} \| \bm{\widetilde{E}} \|_F 
\quad (\text{Lemma}~\ref{lemma:nuclear_opnorm},\text{Lemma}~\ref{lemma:nuclear_fnorm})
\\
&\leq \|\bm{\widetilde{Z}}\|_{F} \cdot \|\bm{\widetilde{U}} \|_*
+ \sqrt{2n} \| \bm{\widetilde{E}} \|_F 
\quad (\text{Lemma}~\ref{lemma:op_fnorm})
\\
&\leq \|\bm{\widetilde{Z}}\|_{F} \cdot \left( \sqrt{2n} \sqrt{\overline{\cos\theta}} + \sqrt{2n(2n - 1)} \cdot \sqrt{1 - \overline{\cos\theta}} \right) \\
&+ \sqrt{2n} \| \bm{\widetilde{E}} \|_F 
\quad (\text{Lemma}~\ref{lemma:nuclear_cos})
\\
&\leq \|\bm{\widetilde{Z}}\|_{F} \cdot \left(\sqrt{2n} \sqrt{ \ell {\cos\tilde{\theta}}} + \sqrt{2n(2n - 1)} \cdot \sqrt{1 - {\cos\tilde{\theta}}}\right) \\
&+ \sqrt{2n} \| \bm{\widetilde{E}} \|_F 
\end{align*}

Putting everything together, we derive:
\begin{align*}
&\text{OLE}(\bm{M}_0, \bm{M}_1) \geq \sum_{c \in \{0,1\}} 
\left( 
\|\bm{Z}_c\|_F - \sqrt{n}\|\bm{E}_c\|_F
\right)\\
&- (
\|\bm{\widetilde{Z}}\|_{F} \cdot (\sqrt{2n} \sqrt{ \ell {\cos\tilde{\theta}}} + \sqrt{2n(2n - 1)} \cdot \sqrt{1 - {\cos\tilde{\theta}}}) \\
&+ \sqrt{2n} \| \bm{\widetilde{E}} \|_F
    ) 
\end{align*}

Denote $\gamma(n) = \sqrt{2} \cdot \frac{ \Gamma\left( \frac{n+1}{2} \right) }{ \Gamma\left( \frac{n}{2} \right) }$, where $\Gamma(\cdot)$ is the Gamma function. According to Lemma~\ref{lemma:exp_fnorm}, we have:

\begin{align*}
&\mathbb{E}(\text{OLE}(\bm{M}_0, \bm{M}_1)) 
\geq 2\gamma(rn) - 
2\sqrt{n}\gamma(dn) \\
&\qquad -2\gamma(rn) \cdot \left(\sqrt{2n} \sqrt{ \ell {\cos\tilde{\theta}}} + \sqrt{2n(2n - 1)} \cdot \sqrt{1 - {\cos\tilde{\theta}}}\right) \\
&\qquad - \sqrt{2n} \cdot \gamma(2dn) \\
&= \underbrace{2\gamma(rn) - 
2\sqrt{n}\gamma(dn) -
\sqrt{2n} \gamma(2dn)}_{C_1} \\
& \qquad - \underbrace{2\gamma(rn)}_{C_2} \cdot
\left(
\underbrace{\sqrt{2n} \sqrt{ \ell {\cos\tilde{\theta}}} + \sqrt{2n(2n - 1)} \cdot \sqrt{1 - {\cos\tilde{\theta}}}}_{\phi(\widetilde{\theta})}
\right)
\end{align*}

Over the interval $[0, \arccos \frac{1}{2n}]$, when $\tilde{\theta}$ decreases, which indicates that subspaces of each class are less orthogonal, the function $\phi(\tilde{\theta})$ is strictly decreasing. Therefore, the right-hand side lower bound is increasing. When $n \to \infty$, $\arccos \frac{1}{2n} \to \frac{\pi}{2}$, hence this conclusion holds for most $\theta$.

\end{proof}

\section{Proof of Theorem~\ref{theorem:score}}
\label{proof:score}

We first propose the following lemmas:

\begin{lemma}\label{lemma:sherman}
    \textbf{(Sherman–Morrison–Woodbury Formula)}
    $(\bm{A}+\bm{UV^\top})^{-1} = \bm{A}^{-1} - \bm{A}^{-1} \bm{U} (\bm{I} + \bm{V^\top\bm{A}^{-1}\bm{U}})^{-1}\bm{V}^\top\bm{A}^{-1}$
\end{lemma}

\begin{lemma}\label{lemma:det}
    \textbf{(The Matrix Determinant Lemma)} 
    $\det(\bm{A}+\bm{U}\bm{V}^\top) = \det(\bm{I}+\bm{V}^\top\bm{A}^{-1}\bm{U})\det(\bm{A})$
\end{lemma}

Following notations used in Theorem~\ref{theorem:ole_lob}, we further assume the \emph{principal angles} between subspaces \( \text{span}(\bm{U}_0) \) and \( \text{span}(\bm{U}_1) \) are all equal to some angle \( \theta \in [0, \frac{\pi}{2}] \).

We first derive the Bayes optimal classifier for these two distributions. Denote the log-likelihood ratio by:
\[
g(\bm{h}) = \log \frac{p(\bm{h} \mid y = 1)}{p(\bm{h} \mid y = 0)} = \frac{1}{2} \bm{h}^\top (\bm{\Sigma}_0^{-1} - \bm{\Sigma}_1^{-1}) \bm{h} + \frac{1}{2} \log \frac{\det \bm{\Sigma}_0}{\det \bm{\Sigma}_1}
\]

According to Lemma~\ref{lemma:sherman}: 
\[
\bm{\Sigma}_c^{-1}  
= (\sigma^2 \bm{I}_d + \bm{U}_c \bm{U}_c^\top)^{-1} 
= \frac{1}{\sigma^2}\bm{I} - \frac{1}{\sigma^2 (\sigma^2+1)}\bm{U}_c \bm{U}_c^\top
\]
Hence we have:
\[
\bm{\Sigma}_0^{-1} - \bm{\Sigma}_1^{-1}
= \frac{1}{\sigma^2 (\sigma^2+1)} (\bm{U}_1 \bm{U}_1^\top - \bm{U}_0 \bm{U}_0^\top)
\]
According to Lemma~\ref{lemma:det}, $\det(\bm{\Sigma}_0) = \det(\bm{\Sigma}_1)$, hence we have:
\[
g(\bm{h}) = \frac{1}{2\sigma^2 (\sigma^2+1)} \bm{h}^\top(\bm{U}_1 \bm{U}_1^\top - \bm{U}_0 \bm{U}_0^\top) \bm{h}
\]
\[
= C_3 \left( \|\bm{P}_1 \bm{h}\|^2 - \|\bm{P}_0 \bm{h}\|^2 \right),
\]
where $C_3 \coloneqq \frac{1}{2\sigma^2 (\sigma^2+1)}$, $\bm{P}_c := \bm{U}_c \bm{U}_c^\top$. $\bm{U}_c$ for each class can be estimated through Principal Component Analysis (PCA).

The posterior probability on class 1 assigned by the classifier is:
\[
\Pr[y = 1 \mid \bm{h}] := \frac{1}{1 + e^{-g(\bm{h})}} = \varsigma(g(\bm{h})),
\]
where $\varsigma(\cdot)$ is the Sigmoid function.
We define the confidence score as the expected confidence score conditioned on \( y = 1 \):
\[
\xi(\theta) := \mathbb{E}_{\bm{h}|y=1 \sim \mathcal{N}(\bm{0}, \bm{\Sigma}_1})[\varsigma(g(\bm{h}))].
\]

Since \( \bm{h} \sim \mathcal{N}(\bm{0}, \bm{\Sigma}_1) \), we can write:
\[
\bm{h} = \bm{U}_1 \bm{z} + \sigma \bm{\epsilon}, \quad \text{with } \bm{z} \sim \mathcal{N}(\bm{0}, \bm{I}_r),\ \bm{\epsilon} \sim \mathcal{N}(\bm{0}, \bm{I}_d)
\]

The angles of column vectors between $\bm{U}_0$ and $\bm{U_1}$ are all $\theta$, hence we can have 
$\bm{U}_0^\top \bm{U}_1 = \cos \theta \cdot \bm{I}_r$. 
Since \( \bm{P}_0 \bm{h} = \bm{U}_0 \bm{U}_0^\top \bm{U}_1 \bm{z} + \sigma \bm{P}_0 \bm{\epsilon} \), we have:
\[
\|\bm{P}_0 \bm{h}\|^2 = \cos^2 \theta \cdot \|\bm{z}\|^2 + 2 \sigma \cos \theta \cdot \bm{z}^\top \bm{U}_0^\top \bm{\epsilon} + \sigma^2 \|\bm{U}_0^\top \bm{\epsilon}\|^2
\]

Since \( \bm{P}_1 \bm{h} = \bm{U}_1 \bm{z} + \sigma \bm{P}_1 \bm{\epsilon} \), we have:
\[
\|\bm{P}_1 \bm{h}\|^2 = \|\bm{z}\|^2 + 2 \sigma \bm{z}^\top \bm{U}_1^\top \bm{\epsilon} + \sigma^2 \|\bm{U}_1^\top \bm{\epsilon}\|^2
\]

Putting this together:
\begin{align*}
g(\bm{h}) &= C_3 (
\|\bm{z}\|^2 - \cos^2 \theta \cdot \|\bm{z}\|^2 \\
&\qquad+ 
\underbrace{2 \sigma \bm{z}^\top \bm{U}_1^\top \bm{\epsilon} - 2 \sigma \cos \theta \cdot \bm{z}^\top \bm{U}_0^\top \bm{\epsilon}}_{\eta_1}\\
&\qquad+ \underbrace{\sigma^2 \|\bm{U}_1^\top \bm{\epsilon}\|^2 - \sigma^2 \|\bm{U}_0^\top \bm{\epsilon}\|^2}_{\eta_2}
) \\
&= C_3 \left(
 \sin^2 \theta \cdot \|\bm{z}\|^2 + \eta_1 + \eta_2
\right)
\end{align*}

The expectation of \( g(\bm{h}) \) is:
\begin{align*}
\mathbb{E}[g(\bm{h})] &= C_3 \left(
 \sin^2 \theta \cdot \mathbb{E}[ \|\bm{z}\|^2] + \mathbb{E}[\eta_1] + \mathbb{E}[\eta_2]
 \right)
\end{align*}

$\|\bm{z}\|^2$ follows Chi-squared distribution, i.e., $\|\bm{z}\|^2 \sim \chi^2_r$, hence $\mathbb{E}[ \|\bm{z}\|^2] = r$, where $r$ is the length of vector $z$.
Term \( \eta_1 \) involves summation of products of independent zero-mean Gaussian variables, hence has zero mean.
$\|\bm{U}_c^\top \bm{\epsilon}\|^2 \sim \chi^2_r$, hence $\eta_2$ also has zero mean.

Since Sigmoid function \( \varsigma \) is concave on \( \mathbb{R}_+ \), Jensen's inequality~\citep{jensen} implies:
\begin{align*}
    \mathbb{E}[\varsigma(g(\bm{h}))] &\leq \varsigma(\mathbb{E}[g(\bm{h})]) \\
    & = C_3 \cdot \varsigma(\mathbb{E}[g(\bm{h})]) \\
    & = \frac{1}{2\sigma^2 (\sigma^2+1)} \varsigma(r \sin^2 \theta)
\end{align*}

This upper bound is monotonically increasing.
As \( \theta \to 0 \), the subspaces \( \text{span}(\bm{U}_0) \) and \( \text{span}(\bm{U}_1) \) become more aligned, and the upper bound decreases. Therefore, the expected confidence score of the Bayes optimal classifier decreases since the class subspaces become less distinguishable.

\section{Experimental details}\label{appendix:exp}

% Our implementation of the conditional diffusion model is based on the following codebase: \url{https://github.com/VSehwag/minimal-diffusion}, which is released under the MIT license.  
The MNIST dataset is distributed under the Creative Commons Attribution-Share Alike 3.0 license. CIFAR-10 also follows the MIT license. The CelebA dataset is available strictly for non-commercial research purposes.

We use the AdamW optimizer with a learning rate of $10^{-4}$. The number of denoising timesteps for both training and sampling is set to 1,000.
For ACU-SIMS, we select $\omega \in \{0.7, 1.4, 2.1\}$ based on the configuration that yields the lowest FID for each dataset after the first generation. Specifically, we choose $\omega = 0.7$ for MNIST, and $\omega = 1.4$ for both CIFAR-10 and CelebA.
To evaluate precision and recall, we extract feature vectors from synthetic images using a pre-trained VGG-16~\cite{vgg} classifier, following the same setup as in~\cite{improved_metric}. Regarding the neighborhood size $k$, we set $k=20$ for MNIST and $k=10$ for CIFAR-10 and CelebA.

The training, sampling, and evaluation steps for each self-consuming loop are conducted on a NVIDIA RTX A5000 GPU with 24 GB memory. The total execution time per generation varies by dataset: approximately 2 hours for MNIST and 6 hours for CelebA.

\begin{figure*}[htbp]
    \centering
    \includegraphics[width=.65\linewidth]{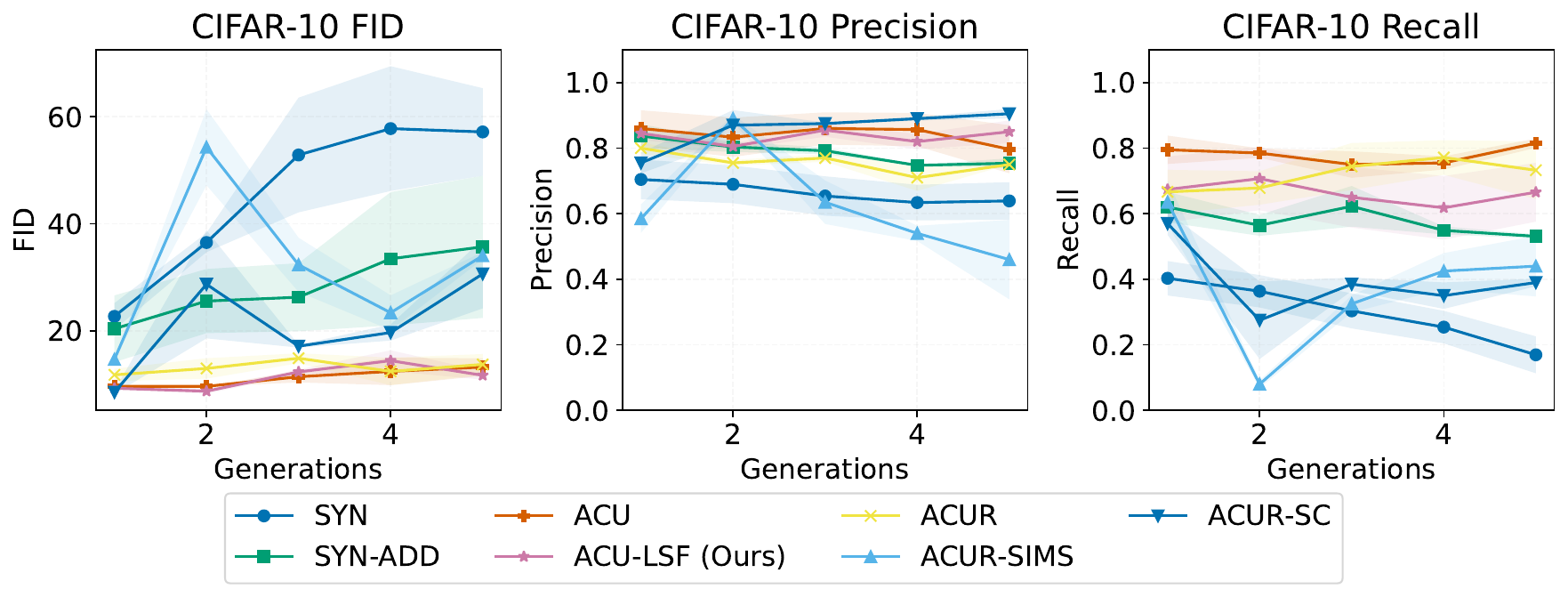}
    \caption{Performance comparison on CIFAR-10 dataset. SYN suffers from model collapse. SYN-ADD partially alleviates it via fresh real data. ACU, ACUR, and ACU-LSF maintain stable metrics across generations. ACU-LSF achieves higher fidelity than ACUR. ACUR-SIMS exhibits instability, while ACUR-SC suffers from reduced recall.}
    \label{fig:cifar10_baselines}
\end{figure*}

\begin{figure*}[htbp]
\centering
\setlength{\tabcolsep}{2pt}
\renewcommand{\arraystretch}{1.4}

\begin{tabular}{@{}>{\centering\arraybackslash}m{1cm} *{4}{>{\centering\arraybackslash}m{3.5cm}}@{}}
  & \textbf{SYN} & \textbf{ACU} & \textbf{ACU-LSF (Ours)} & \textbf{ACU-SC} \\

  \makebox[0pt][c]{\raisebox{1cm}{\rotatebox[origin=c]{90}{\textbf{Generation 1}}}}
 &
  \includegraphics[width=\linewidth]{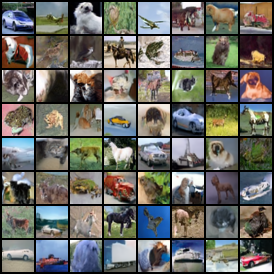} &
  \includegraphics[width=\linewidth]{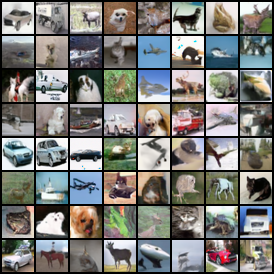} &
  \includegraphics[width=\linewidth]{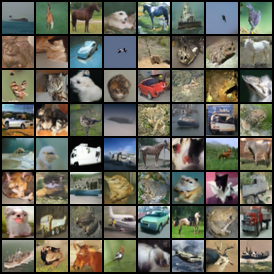} &
  \includegraphics[width=\linewidth]{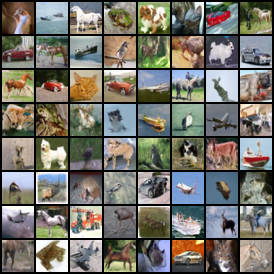} \\

  \makebox[0pt][c]{\raisebox{1cm}{\rotatebox[origin=c]{90}{\textbf{Generation 3}}}} &
  \includegraphics[width=\linewidth]{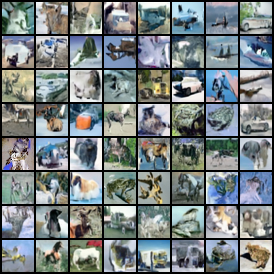} &
  \includegraphics[width=\linewidth]{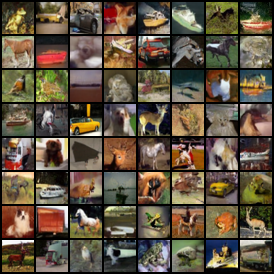} &
  \includegraphics[width=\linewidth]{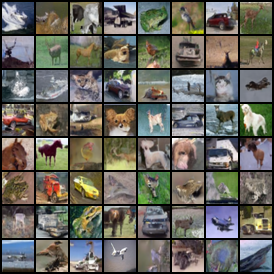} &
  \includegraphics[width=\linewidth]{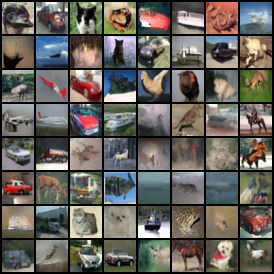} \\

  \makebox[0pt][c]{\raisebox{1cm}{\rotatebox[origin=c]{90}{\textbf{Generation 5}}}} &
  \includegraphics[width=\linewidth]{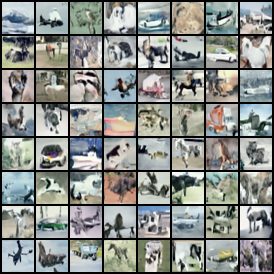} &
  \includegraphics[width=\linewidth]{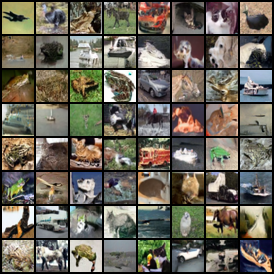} &
  \includegraphics[width=\linewidth]{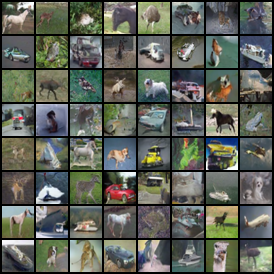} &
  \includegraphics[width=\linewidth]{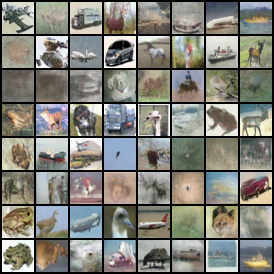} \\
\end{tabular}

\caption{Samples generated by the self-consuming loops of SYN, ACU, ACU-LSF, and ACU-SC across different generations, with the initial model trained on the CIFAR-10 dataset. SYN exhibits model collapse, reflected in declining fidelity and diversity as generation progresses. ACU maintains stable image quality. ACU-LSF achieves image quality comparable to ACU. ACU-SC generates some uniform image content with reduced diversity.}
\label{fig:cifar10_grid}
\end{figure*}

Experimental results on the CIFAR-10 dataset are shown in Fig.~\ref{fig:cifar10_baselines}. We plot the mean and standard error across three runs for each baseline. The trends are consistent with those reported in Section~\ref{sec:exp}. SYN suffers from model collapse, while ACU and SYN-ADD help mitigate it. ACUR-SIMS is less stable due to iterative extrapolation of score functions. ACUR-SC yields high precision but suffers from low recall. Although ACUR achieves higher recall than our method, our approach achieves comparable FID and better precision. Samples generated by the representative self-consuming models are shown in Figure~\ref{fig:cifar10_grid}. SYN exhibits signs of model collapse, i.e., decreasing fidelity and diversity as generation progresses. In contrast, ACU maintains stable image quality across generations. ACU-LSF achieves image quality comparable to ACU. While ACU-SC attempts to mitigate model collapse by relying on image clustering centers, this leads to homogenized image content, compromising overall diversity.

\section{Related works}\label{related}

Recent studies have shown that training generative models on their own synthetic outputs can lead to \textit{model collapse}, characterized by the forgetting of rare samples, degraded fidelity, and reduced diversity~\citep{recrusion, mad}. \citet{recrusion} observe that retraining on synthetic data progressively removes the tails of the real data distribution, a pattern observed across various generative models. Self-consuming training can also amplify pre-existing biases~\citep{news_bias_amplify} or parity gaps in downstream tasks~\citep{FairnessFeedback}, posing risks to fairness.

\citet{mad} term this collapse \textit{Model Autophagy Disorder} (MAD), where artifacts and biases are amplified, and diversity in generated data diminishes. They categorize self-consuming training into three settings: the fully synthetic loop, the synthetic augmentation loop, and the fresh-data loop. Only the latter one can consistently mitigate MAD. They further highlight that sampling bias toward high-quality generations plays a key role in the quality-diversity trade-off. \citet{stable} provide a theoretical guarantee that retraining remains stable if a sufficient portion of real data is maintained in each generation. \citet{recrusive_stability} offer generalization error bounds, showing that both model architecture and the real/synthetic data ratio impact retraining outcomes. \citet{remove} analyze real-synthetic interactions and find that discarding synthetic data can be more beneficial than merely adding additional real samples.

Beyond maintaining fresh real data, several alternative mitigations have been proposed. \citet{accumulate} accumulate all historical data during training, supported by theoretical analysis under a linear regression setting. Other methods rely on external guidance signals. \citet{sims} propose Self-IMproving diffusion models with Synthetic data (SIMS), where the drift of the retrained model negatively guided the score function closer to the real data distribution. \citet{verification} suggest that high-quality synthetic data alone can be sufficient for retraining if verified and selected appropriately. \citet{detection_model_collapse} train a machine-generated text detector and apply importance sampling to reduce model collapse. In contrast, our work leverages latent representations in diffusion models as a simpler detector. \citet{self_correction} introduce a self-correction mechanism to map synthetic samples to more plausible real-distribution counterparts. For instance, in motion synthesis, integrating physical laws yields greater stability compared to purely retraining with synthetic data.

Recent studies have also offered valuable insights into the training dynamics and latent representations of diffusion models. \citet{self_consuming_diffusion} model the impact of real data proportion in retraining diffusion models with a one-hidden-layer network as the score function. \citet{diff_cluster} demonstrate that training diffusion models under specific assumptions about training data distribution and model parameterization is equivalent to subspace clustering. \citet{semantic_latent} show that latent features of frozen diffusion models exhibit desirable properties like linearity, enabling semantic image editing. \citet{dynamic} investigate how latent representations evolve during denoising, revealing a trade-off between representation quality and generation quality. \citet{hallucination} observe that hallucinations in generated samples can be identified by high variance in the sampling trajectory, and that removing such hallucinations helps stabilize recursive training. However, variance-based filtering requires access to the sampling trajectory of each image, which is often unavailable in practical settings.

\section{Conclusions and limitations}
\label{appendix:conclusion}
As generative models increasingly rely on mixtures of human- and model-generated data, concerns arise regarding potential degradation in model quality and diversity. In this work, we address the phenomenon of model collapse caused by self-consuming training loops and propose a filtering algorithm to remove unrealistic synthetic images, thereby stabilizing the retraining process.

Our key insight is that the low-dimensional structure of latent representations extracted by diffusion models deteriorates over successive generations. We leverage this observation to assess the misalignment between collected training samples and the underlying real data distribution. By filtering out samples with poor alignment, we are able to improve stability during retraining.
We introduce a unified theoretical framework to support our observations and approach. We validate the effectiveness of our approach through empirical results on real-world datasets. Under a fixed training budget, our method  outperforms competitive baselines in image quality and stability.

\textbf{Limitations.} One limitation of our work is the assumption that image classes remain fixed throughout the self-consuming process. However, in real-world scenarios such as continual learning or unlearning, models may be retrained to acquire new classes or deliberately forget existing ones. These dynamic changes can induce distribution shifts in the latent representations, potentially affecting the stability and generalization of filtering with probing classifier. Extending our analysis to settings where the number of classes evolves over time presents a valuable direction for future research. 

\end{document}